\pdfoutput=1
\documentclass{article}
\usepackage[utf8]{inputenc}
\usepackage{fancyhdr}
\usepackage[english]{babel}

\usepackage{amsmath}
\usepackage{amssymb}
\usepackage{ amssymb }
\usepackage{amsthm}
\usepackage{changepage}
\usepackage{ dsfont }
\usepackage{graphicx,wrapfig,lipsum}
\usepackage{graphicx}
\usepackage[utf8]{inputenc}

\usepackage[labelfont=bf]{caption}
\usepackage{subcaption}
\usepackage{listings}
\usepackage{color}

\definecolor{codegreen}{rgb}{0,0.6,0}
\definecolor{codegray}{rgb}{0.5,0.5,0.5}
\definecolor{codepurple}{rgb}{0.58,0,0.82}
\definecolor{backcolour}{rgb}{0.95,0.95,0.92}

\lstdefinestyle{mystyle}{
    backgroundcolor=\color{backcolour},
    commentstyle=\color{codegreen},
    keywordstyle=\color{magenta},
    numberstyle=\tiny\color{codegray},
    stringstyle=\color{codepurple},
    basicstyle=\footnotesize,
    breakatwhitespace=false,
    breaklines=true,
    captionpos=b,
    keepspaces=true,
    numbers=left,
    numbersep=5pt,
    showspaces=false,
    showstringspaces=false,
    showtabs=false,
    tabsize=2
}

\theoremstyle{plain}
\newtheorem{theorem}{Theorem}

\newtheorem{definition}{Definition}

\title{\textbf{Learning model-based strategies in simple environments with hierarchical q-networks}}
\author{Necati Alp Müyesser \\
        \small{Department of Mathematical Sciences} \\
        \small{Carnegie Mellon University}\\
        \texttt{\small{nmuyesse@andrew.cmu.edu}}
        \and
        Kyle Dunovan \\
        \small{Department of Psychology}\\
        \small{Carnegie Mellon University}\\
        \texttt{\small{kdunovan@andrew.cmu.edu}}
        \and
        Timothy Verstynen\\
        \small{Department of Psychology}\\
        \small{Center for the Neural Basis of Cognition}\\
        \small{Carnegie Mellon University} \\
        \texttt{\small{timothyv@andrew.cmu.edu}}}

\begin{document}

\maketitle

\begin{abstract}
Recent advances in deep learning have allowed artificial agents to rival human-level performance on a wide range of complex tasks; however, the ability of these networks to learn generalizable strategies remains a pressing challenge. This critical limitation is due in part to two factors: the opaque information representation in deep neural networks and the complexity of the task environments in which they are typically deployed. Here we propose a novel Hierarchical Q-Network (HQN), motivated by theories of the hierarchical organization of the human prefrontal cortex, that attempts to identify lower dimensional patterns in the value landscape that can be exploited to construct an internal model of  rules in simple environments. We draw on combinatorial games, where there exists a single optimal strategy for winning that generalizes across other features of the game, to probe the strategy generalization of the HQN and other reinforcement learning (RL) agents using variations of Wythoff’s game. Traditional RL approaches failed to reach satisfactory performance on variants of Wythoff’s Game; however, the HQN learned heuristic-like strategies that generalized across changes in board configuration. More importantly, the HQN allowed for transparent inspection of the agent’s internal model of the game following training. Our results show how a biologically inspired hierarchical learner can facilitate learning abstract rules to promote robust and flexible action policies in simplified training environments with clearly delineated optimal strategies.
\end{abstract}
\newpage

\section{Introduction}

\par Deep reinforcement learning (DRL) networks currently rival human-level performance in a variety of domains, including object recognition \cite{objectrec}, speech recognition \cite{speechrec}, video games \cite{atari}, and complex board games such as Go \cite{alphago}. Despite the impressive achievements of DRL, networks fail to adapt to trivial changes of the inputs and goals of the learning task, such as changes to board dimensions and structure \cite{lake2017building}. One potential reason for this shortcoming is that DRL algorithms learn through extensive feedback about the value of specific input-output associations, without any appreciation for the organizing features of the game that govern these associations. In contrast, evidence from cognitive science suggests that humans learn to perform complex tasks through a model-based approach that involves constructing an internal model of the organizing principles or rules of the environment \cite{lake2017building, toyama2017simple}. This form of learning is particularly important in dynamic environments where survival depends on the ability to generalize previous training to novel settings \cite{kool2016does}. While model-based learning algorithms stand to improve the robustness of DRL networks in dynamic environments, it remains a largely unanswered question how this might be achieved.

\par One reason for the paucity of model-based approaches to deep learning is that DRL agents are typically developed to solve highly complex tasks, thereby precluding any straightforward process for exploring possible internal models of the environment\cite{marcus2018deep}. One way to facilitate development of deep model-based learning agents is to identify simpler testing environments that effectively reduce the number of available features from which internal models can be constructed. A recent paper \cite{leike} highlighted several advantages of simplified environments for evaluating the safety and robustness of DRL agents, showing that simple "gridworld" environments provide a tractable way for identifying pitfalls in the learned action policy. Indeed, in these simple environments two state-of-the-art DRL networks failed to effectively adapt to subtle differences between training and testing environments, highlighting the need for more robust RL and DRL algorithms. Another important environmental characteristic for the purposes of building model-based agents is the ability to alter, remove, or introduce dimensions of the environment without rendering previous training irrelevant. In other words, in order to fairly evaluate the success of the agent, there needs to exist a reliable strategy or internal model that is robust across variations of the environment or task rules \cite{lieder2015use}.

\par The conditions described above are satisfied by a class of impartial combinatorial games \cite{berlekamp}. The most notable difference between an impartial game and a game such as Go is that most impartial games have a \textit{ground truth} solution. Every position in an impartial game is either $\textit{hot}$, meaning that there exists a winning strategy for the player about to make a move, or $\textit{cold}$, meaning that under optimal play, the player about to make a move will always lose. The distribution of $\textit{hot}$ and $\textit{cold}$ positions across the state space in an impartial game usually comes with inherent mathematical structure. In an impartial game, Player $1$ ($p_1$) and Player $2$ ($p_2$) alternate in making moves until there are no available moves to make, with the player to make the last move declared the winner. The function that takes in a state and returns the set of available legal actions has as its domain an infinite set, meaning it can be generalized to arbitrary dimensions, lending impartial games particularly amenable to a model-based learning strategy.

\par Working within the constraints of impartial game theory, we now consider the differences between RL agents with model-free and model-based learning strategies. Substantial evidence from cognitive psychology and neuroscience suggests that model-based learning is associated with hierarchical information processing, with action-value associations learned at lower levels of the hierarchy and abstract predictions about the environment at higher levels \cite{doll12, smittenaar, wunderlich, doll, russek17, odoherty17}. One example of such a hierarchy is the prefrontal cortico-basal ganglia (BG) network \cite{frank, badre2009} found in many mammalian species. Converging evidence from human and animal neurophysiological experiments shows that the prefrontal BG networks engage in model-free learning of action-values, driven by phasic dopaminergic signals from the midbrain that alter the weights of cortical inputs to the BG in accordance with environmental feedback \cite{schultz97, eshel2015arithmetic, eshel2016dopamine}. As a result of this plasticity, rewarded (or punished) actions become more (or less) likely to be executed in the future. This form of learning is analogous to that enacted by deep Q-Learning (DQL) agents that exhibit behavioral policies determined solely by the feedback of previous actions.

\par Importantly, rather than arising from an entirely separate and independent process, model-based learning can be viewed as a companion system that is both informed by and exerts control over the feedback-dependent associations formed through model-free learning \cite{odoherty17}. Behavior becomes “model-based” when lower-level feedback dependent representations are leveraged to construct an internal model of the environmental dynamics responsible for previous observations \cite{russek17}, sometimes referred to as a “generative model”. A key difference in the behavioral outcomes of these two forms of learning is flexibility \cite{daw2005, odoherty17}: model-free learning results in habitual actions based on a static cache of associated values whereas model-based learning results in goal-directed actions based on inferred dynamics of the environment. Evidence from human neuroimaging experiments suggests that the shift from model-free to model-based policies is driven by a concomitant shift from BG to prefrontal behavioral control \cite{daw2005, badre}, signaling a shift away from feedback-dependent knowledge to active predictions drawn from the agent’s internal model of the environment.

\par Motivated by the hierarchical organization of prefrontal cortico-BG systems that are thought to implement model-based learning in the human brain, we devised a novel Hierarchical Q-Network (HQN) that attempts to build an internal strategy (e.g., generative model) based on inferred patterns of hot and cold positions (e.g., model-free Q-learning) on a variant of impartial combinatorial games called Wythoff’s game. We show how this hierarchical learning structure promotes generalizability and robustness to rule changes while also improving post-training interpretability of learning outcomes. Compared to the performance of standard Q learning and deep Q-Networks, the HQN is markedly faster at learning the task and, more importantly, shows clear benefits to the transfer of learning, not only to alterations of Wythoff’s game, but across a variety of other impartial games with distinct, but similar rule structures. Below, we describe our findings, highlighting 1) the benefits afforded by impartial games for developing more robust deep learning agents and 2) the importance of hierarchical learning in environments that demand flexibility.

\begin{figure}[t]
\includegraphics[width=12cm]{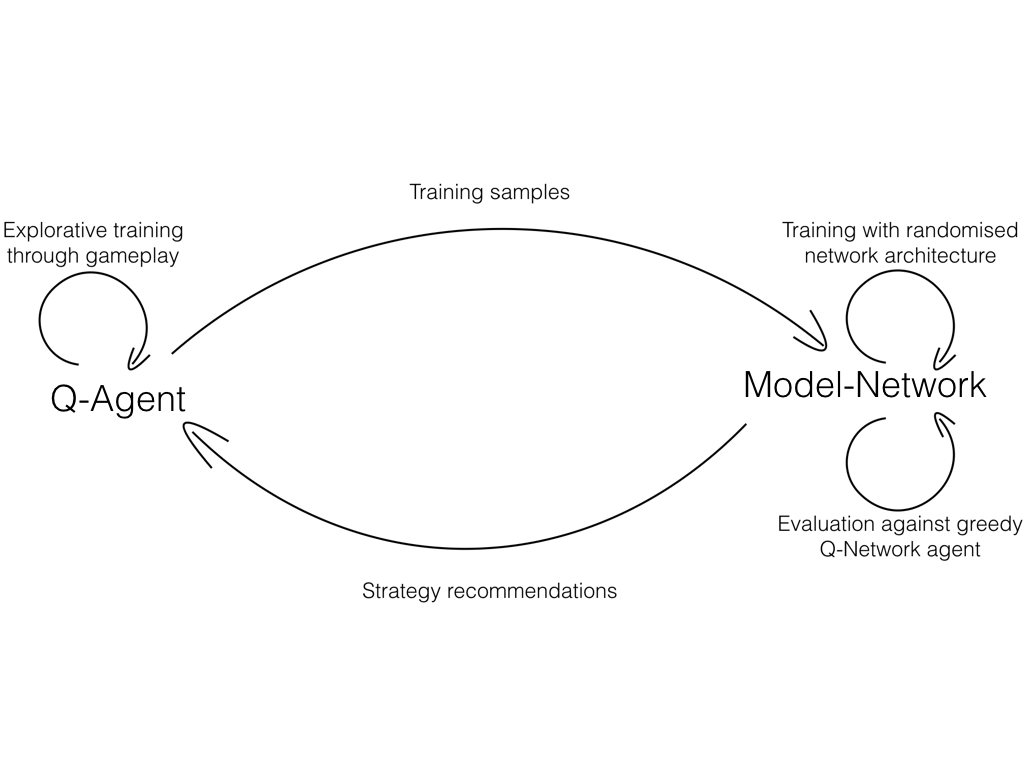}
\centering
\caption{The proposed model-based learning architecture: The Hierarchical Q-Network (HQN). The Q-Network and the Model Network use datasets that differ in dimensionality, but cooperate in order to generate a generalized model for the impartial game
environment.}
\end{figure}

\section{Methods}

\subsection{Impartial games: Wythoff's game, Nim, and Euclid}
Wythoff's game is played on a two dimensional grid in which players alternate turns to move an object that is initially on the bottom-right corner towards the top-left corner. The player who gets to place the object in the top-left corner terminates, and thereby wins the game. Every turn, the object can be moved horizontally, vertically, or diagonally towards the top-left corner.

\begin{definition}
Wythoff's game is an impartial game where the states are all 2-dimensional non-negative integer coordinates. From coordinates $(a, b)$, $p_1$ and $p_2$ can access all states of the form $(i, b)$, $(a, j)$, and $(a-k,b-k)$ where $0<i<a$, $0<j<b$, and $0<k<min(a,b)$.
\end{definition}

As mentioned above, every position in an impartial game is either $\textit{hot}$ or \textit{cold}, indicating whether $p_1$ or $p_2$ will win the game under optimal play. For formal definitions of $\textit{hot}$ and \textit{cold} positions, see Definition \ref{def:hotcold}, for an inductive proof of the partitioning, see Theorem \ref{thm:hotcold}.

\par The partition of $\textit{hot}$ and \textit{cold} positions in Wythoff's game is deeply embedded in properties of the Fibonacci string and the golden ratio \cite{berlekamp}.

\begin{theorem}\label{thm:wythoff}
  Let $r_k = \lfloor k \phi \rfloor$ and $c_k = \lfloor k \phi^2 \rfloor$ where $\phi=\frac{1+\sqrt{5}}{2}$ is the golden ratio. Then, all cold positions in Wythoff's game is in the form $(r_k, c_k)$ or $(c_k, r_k)$, where $k$ is a natural number.
\end{theorem}

The mathematical structure of Wythoff's game (expressed by Theorem \ref{thm:wythoff}) manifests in a highly patterned separation of $\textit{hot}$ and $\textit{cold}$ positions (see Figure \ref{fig:hotcold}).

\begin{figure}[htp]
\includegraphics[width=8cm]{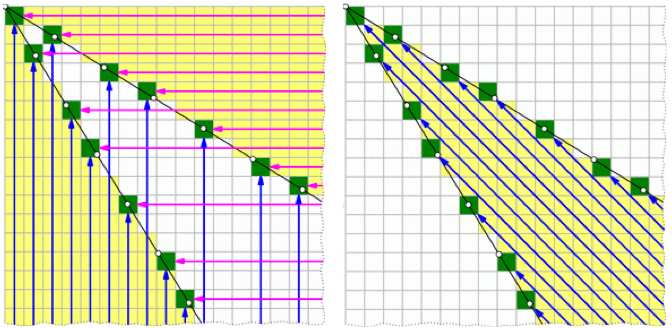}
\centering
\caption{Cold positions in Wythoff's game are distributed along two symmetrical lines. Arrows show how from every other position (hot) there exists a move to a cold position. There does not exist any move from a cold position to another cold position. Figure used with permission from Zachary Abel \cite{zach}}
\label{fig:hotcold}
\end{figure}

\par While we benchmark our HQN agent on Wythoff's game, we will also subject the game to certain rule changes in later sections. The $\textit{hot}$-\textit{cold} partition of the resulting games are structurally similar to Wythoff's game, and are discussed in the Appendix section. See Figure~\ref{fig:nimeuclid} for a visualization.
\begin{definition}
  We denote by Nim the impartial game resulting from a Wythoff's game where diagonal moves are disallowed.
\end{definition}
\begin{definition}
  We denote by Euclid the impartial game resulting from a Nim where a distance travelled in the horizontal or vertical direction has to be a multiple of the minimum of the horizontal and vertical distance to the top-left corner.
\end{definition}

\subsection{Hierarchical Q-Network (HQN)}

\subsubsection{Overview}

The HQN is comprised of two interconnected systems, the Q-agent and the Model-Network, that attempt to cooperatively generate an internal model for the task environment while working with datasets that differ in dimensionality. The Q-agent works with a high-dimensional dataset reflecting the expected value of state-action pairs. The Model-Network, on the other hand, feeds off of the conclusions of the Q-agent obtain a low-dimensional dataset that solely reflects the expected values of given states. The Model-Network uses a deep neural network to extrapolate a model from the extracted dataset and evaluates the value of the model by testing the model against an opponent simulated by the Q-agent. In return, the Model-Network biases the action policy of the Q-agent to favor movements to states more likely to generalize to larger environments. The behavior of the Q-agent then effectively explores state-action pairs that contradict or corroborate the current generative model of the Model Network. The HQN succeeds if and only if the Model-Network converges on a generalizable model of the given environment.

\subsubsection{Network Details}
A summary of the underlying logic behind the HQN is given above, whereas a detailed discussion about its implementation including pseudocode is provided below. Here, we provide details about the architectures of the networks that compose the HQN.
\par The Q-agent component of the HQN uses Q-Learning \cite{suttonbarto} to build estimates for how good a given state is for the player, based solely on gameplay experience. The learning rate ($\alpha$) is set to $0.1$ and the discount rate ($\lambda$) is set to $1$. Action selection is randomized through a Boltzmann distribution where the exploration constant ($\beta$) is set to $0.7$. Further details are given in the next section.

\par The Model-Network is a feed-forward, single-layer ($15$ neurons), and fully-connected network. The error limit and maximum number of iterations are also randomized, to account for errors due to over-fitting or under-fitting. Sigmoid-activation function is used as the activation function for individual neurons. We use the standard backpropogation algorithm \cite{rumelhart} to train the network with the cross-entropy function as the cost-function to avoid learning slow-down, as dysfunctional models are expendable, making the trade-off worthwhile. The cross-entropy cost-function is given by: $$ Cost=-\frac{1}{n}\sum_{x}[y\ln{a} + (1-y)\ln{(1-a)}]$$
Where $x$ is over all training inputs, $n$ is the number of inputs, $y$ is the desired output, and $a$ is the output of the neuron.
\par Nimblenet \cite{nimble} library for Python was used in order to simulate the neural network that is in the architecture of the Model-Network. Further details about the separate networks and their interaction is given in the next section.

\subsubsection{The Q-agent}

\par The Q-agent relies on Q-Learning, a standard model-free reinforcement learning technique \cite{suttonbarto}, to estimate the expected value (Q-value) of a state-action pair in a given environment over multiple training sessions (see Figure \ref{fig:QLearnPython}). Every move made adjusts the values stored in the Q-table through value iteration update. Directly updating  state-action pairs in this way affords greater precision, and is computationally simpler, than relying on error propagation to adjust the weights of neural network. The Q-Network was able to achieve similar performance to the basic Q-agent when action values were estimated independent of the current state (e.g., board position), albeit less efficiently.

\par Here $(s,a)$ is a state-action pair, while $s'$ is the new state after action $a$ is taken at state $s$, $r_{(s,a)}$ is the reward associated with $(s,a)$, and $\alpha$ and $\lambda$ are the learning and the discount rate, respectively. Finally, $moves(s)$ is the function that returns the set of available actions from state $s$ in the environment. Then, the Q-value is updated through: $$Q(s,a):= Q(s,a) + [\alpha\cdot(r_{(s,a)} +\lambda\cdot\max_{a'\in moves(s')}(Q(s',a')) - Q(s,a)]$$

\par We take $\alpha=0.1$ and $\lambda=1$, since every action has equal effect on the outcome in a given impartial game, so discounting future rewards is redundant.

\par Actions are selected through the Boltzmann distribution that uses current approximations of the Q-Values to generate a weighted probability space.

\par Explicitly, the probability that action $a\in moves(s)$ is selected is given by: $$\frac{e^{\beta\cdot Q(s,a)}}{\sum_{a'\in moves(s)}{e^{\beta\cdot Q(s,a')}}} $$

\par Here $\beta$ is the constant that determines how exploratory or exploitative the action selection process is going to be. We set $\beta=0.7$ throughout, reflecting a moderate degree of exploratory behavior in the model.

\par It is important to note that the process that selects actions to explore also depends on the \textit{model} as a variable. Actions favored by the $model$ generated by the Model-Network get a boost in their probability of being selected. Before the decision process is left to the Boltzmann probability space, the HQN decides whether to explore the action recommended by the \textit{model} with probability: $$L - e^{-\zeta\cdot\mathds{E}[model]} $$
$\mathds{E}[model]$ is the expected value, or performance of the \textit{model} as computed by the Model-Network. $L$ is the limit imposed on the confidence on the \textit{model}, in order to maintain that the Q-agent still operates mostly independently of the Model-Network. Otherwise, ``echo-loops'' may be created in the HQN. $\zeta$ is the
steepness factor, determining how fast the probability approaches the limit $L$ as $\mathds{E}[model]$ increases. We set $L=0.25$ as an appropriate limit, and $\zeta=7$ to get an optimal probability function with respect to $\mathds{E}[model]$, where the model does not begin to influence the Q-agent until $\mathds{E}[model]>0.25$.

\begin{figure}
\begin{lstlisting}[language=Python]
  def Q_learn(Q, trials, lr, y, beta, (rows, cols), bestModel):
      for episode in range(trials):
          #initialize random game given dimensions
          game.row, game.col = randint(0, rows), randint(0, cols)
          while not game.isTerminal():
              state = game.getIndex()
              #noisy decision process for action selection:
              if random.random()<getModelConf(zeta, lim, model):
                  #take tip from model to select action
                  action = getModelDecision(model, game)
              else:
                  #or explore through weighted probability dist:
                  pSpace = makeBoltzmannPspace(Q, game, beta)
                  action = weighted_choice(pSpace) #select action
              game.makeAction(action)
              if game.isTerminal(): reward = 1 #action wins game
              elif not game.isTerminal():
                  makeGreedyQMove(Q, game) #simulated opponent move
                  if game.isTerminal(): reward=-1 #opponent wins
                  else: reward = 0 #game goes on
              #get best obtainable Q-value greedily:
              nextQValue = max(Q[game.getIndex()].values())
              #update Q-Table:
              Q[state][action] = lr * (reward + y*nextQValue
                                      - Q_dict[state][action])
\end{lstlisting}
\caption{Python Code for Q\_learn, the learning algorithm used by the Q-agent}
\label{fig:QLearnPython}
\end{figure}

\subsubsection{The Model-Network}

\par The question that motivates the Q-agent is ``What moves should I make in which positions to maximize my likelihood of winning?''. However, this question is restricted to the space in which learning occurs. Thus, the question that motivates the Model Network is ``Are some positions better for me than others, if so, is there any \textit{structure} to how these positions are distributed across the board?'' For an $n$ by $m$ Wythoff's game, there are $2^{nm}$ possible ways in which good and bad positions could be distributed. But without the latter question, the former question seems too short-sighted in order to yield any useful insights into the nature of the game. The HQN architecture allows us to ask these questions simultaneously.

\par While the Q-agent attempts to approximate the Q-values of state-action pairs, the Model-Network works with simply the expected values of individual states in order to find a heuristic that will separate good states from bad states (see Figure \ref{fig:ModelLearnPython}). The expected value of a state is simply the Q-value of the best available action from that state. $$\mathbb{E}[state]= \max_{a\in moves(s)}Q(s,a)$$ The Model-Network, equipped with some fixed confidence threshold, $\epsilon$ creates a dataset classifying state $s$ as \textit{cold} if $0\leq\mathbb{E}[state]\leq\epsilon$ and \textit{hot} if $1-\epsilon<\mathbb{E}[state]\leq 1$. Random samples of this dataset are then fed into a neural network. We refer to the trained neural network as the \textit{model}. Architectural details about the neural network was given in the previous section.

\par The Model-Network evaluates the performance of a model by benchmarking against a greedy-agent that has access to the Q-agent, as opposed to perfect-play, thus the training process remains unsupervised. The model receives a performance score between $0$ and $1$, based on the ratio of games the model can win against the greedy-Q-agent. Since the Q-agent almost always remains more accurate on smaller board sizes, benchmarking games are played on a larger board size so as to favor potentially generalizable models.

\begin{figure}[h]
\begin{lstlisting}[language=Python]
  def modelBuild(Q, conf):
      Q_part = getQPartition(Q, qDim) #obtain E[s] values
      dataset = getModelDataSet(Q_part, conf) #noisy extraction
      modelNet = NeuralNet(networkSettings) #set up neural-net
      #train network with the extracted dataset:
      backpropagation(modelNet, dataset, cross_entropy_cost,
                      error_limit, max_iterations)
      #calculate model performance vs. greedy Q-agent:
      perf = evaluateModel(modelNet, Q, trials, game=Wythoff)
      return modelNet, perf
\end{lstlisting}
\centering
\caption{Code for modelBuild, learning algorithm used by the Model-Network}
\label{fig:ModelLearnPython}
\end{figure}

\subsubsection{Q-agent and Model-Network}

\begin{figure}[h]
\begin{lstlisting}[language=Python]
  def modelBasedLearn(Q, trials, game, ...):
      global bestModel, bestPerformance
      #re-evaluate current best model:
      currentPerf = evaluateModel(bestModel, Q, trials, game)
      if severePerformanceDrop(currentPerf, bestPerformance):
          modelMemory.add(bestModel) #store old high-perf model
          #check if any models in memory fit the task:
          bestModel = rememberModel(modelMemory, game)
          bestPerformance = evaluateModel(bestModel, .. )
          Q = reset_Q_agent(...)
      #train the Q-agent:
      Q_learn(Q, trials, bestModel, lr, y, beta, zeta, ...)
      #then, train a new model:
      (newModel, newPerformance) = modelBuild(Q, conf, ...)
      #compare new model to the past best performing model:
      if newPerformance>=bestPerformance:
            bestModel = newModel
            bestPerformance = newPerformance
\end{lstlisting}
\centering
\caption{pseudo-code for modelBasedLearn, the learning algorithm used by the HQN-agent. Q\_learn and modelBuild are the two main components of the algorithm}
\end{figure}

\par The full HQN integrates both the Q-agent and the Model-Network through the learning algorithm \textit{modelBasedLearn}. On every iteration of the \textit{modelBasedLearn} algorithm, the Q-agent is trained on the specified amount of gameplays, and the process concludes with the construction of a candidate model, potentially replacing the current best-performing model. Note that even models that eventually get outplaced have a positive impact on the learning outcomes, since \textit{hypotheses} from flawed models get contradicted by the Q-agent, allowing for the construction of more accurate models in upcoming iterations.

\par Performance of the HQN agent was also tested against changes in the rules of the game, without explicitly notifying the agent of such changes. It is crucial that the HQN agent is able to detect such changes and adapt to the new rules of the game, especially if the HQN agent was trained on the same set of rules earlier. In order to do so, we allow the HQN agent access to a dataset consisting of the calculated performances of the current best performing model. For a predetermined $\Delta$, if $\frac{currentPerf}{avgPerf}<\Delta$, where $avgPerf$ is the average of the past performances, and $currentPerf$ is the current calculated performance of the model, HQN detects a severe performance drop. In this case, the $bestModel$ is stored away on the $modelMemory$ if the need for that same model later arises. $modelMemory$ is also checked for the existence of models that would fit the new rules of the $game$, and if so, that model is used as the $bestModel$ variable.

\subsubsection{Model-Free Learning Agents}
We compared the HQN to two non-hierarchical implementations of Q-learning.
\par \textbf{Q-Agent}
\par We benchmark the HQN against an independent Q-Agent to illustrate the effect of the addition of the Model-Network to the system. The Q-Agent has an almost identical framework to the Q-Agent component of the HQN agent. The only difference is that this Q-agent does not have its exploration procedure influenced by a Model-Network. Hence, we do not include more details about its implementation.
\par \textbf{Q-Network}
\begin{figure}
  \includegraphics[width=12cm]{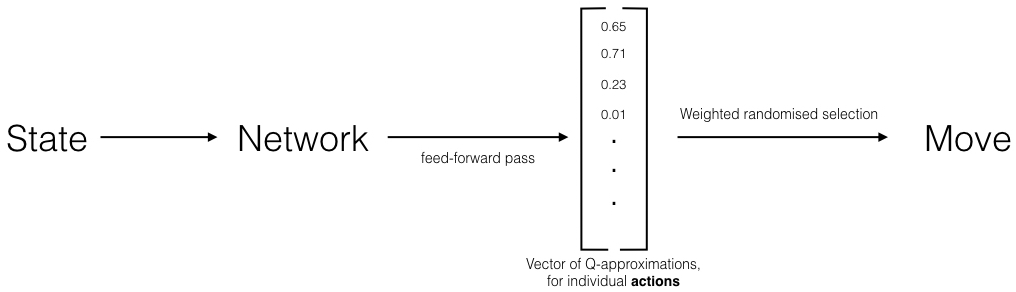}
  \caption{The initial stages of the training of the Q-Network. Given a state, the network attempts to produce a vector of values that represent the approximated reward of each possible action. While testing performance, the highest valued action will be greedily chosen. During training, action that is chosen will depend on how explorative the Boltzmann distribution is. Afterwards, weights are readjusted to account for the reward error via the backpropagation algorithm.}
\label{fig:qnetworkalgorithm}
\end{figure}
\begin{figure}[h]
\begin{lstlisting}[language=Python]
  def Q_network_learn(network, ...):
    (game.row, game.col) = (randint(0, rows), randint(0,cols))
    while not game.isTerminal():
      #Feed-Forward Pass:
      q_estimates=network.predict(game.getInstance())[0]
      #noisy Boltzmann decision process for action selection:
      pSpace = makeBoltzmannPspace(network, game, beta)
      move = weighted_choice(pSpace)
      #simulate game-play:
      game.makeAction(move) #agent's action
      if game.isTerminal():
        reward = 1
      else:
        makePerfectWythoffMove(game)
        if game.isTerminal():
          reward=-1
        else:
          #find reward for greedy move from new state:
          best_q = get_best_q(network, game)
          reward = y*best_q #discounted reward
      #update the Q vector from state accordingly
      q_estimates[move] += lr * (reward -  q_estimates[move])
      dataset=[Instance(state , q_estimates)] #resulting dataset
      #adjust neuronal weights based on the adjusted q_vector:
      backpropagation(network, dataset,cross_entropy_cost,
                                          max_iterations=1)

\end{lstlisting}
\centering
\caption{pseudo-code for Q\_network\_learn, the learning algorithm used by the Q-Network agent. The QN agent was not able to perform more accurately than random chance within the given time constraints of the benchmark.}
\label{fig:qnetwork}
\end{figure}
\par The core difference between the Q-Network and the Q-Agent is that the Q-Network makes use of a neural network to approximate the Q-function, whereas the Q-agent algorithm does not attempt to make an inference beyond the look-up-table process for the Q-values of the state-action pairs. A high-level explanation of the algorithm is given in Figure \ref{fig:qnetwork}, and detailed pseudo-code is given in Figure \ref{fig:qnetwork}. Nimblenet \cite{nimble} was used to simulate the neural network.
\par The network was fully-connected and single-layer, however, more layers did not have a significant effect on the learning outcomes. The standard backpropogation algorithm was used with the sum-squared-error cost function with sigmoid activation function on the individual units. $\alpha=0.01$ and $\beta=0.7$.

\section{Results}

\subsection{Model Building}

The Model-Network displayed great efficacy in producing generalizable models for Wythoff's game and its variants. Figure~\ref{fig:ModelStages} shows how the two components of the HQN learned the value of board positions at different stages of learning. Models that are developed in the earlier stages of training remain mostly irrelevant to generalization; however, models that \textit{meaningfully} generalize, although with low accuracy accurately, begin to emerge soon after initial training. Such models are crucial for the learning process because they influence the way the Q-Network chooses to explore different action spaces. Without such guidance, the Q-Network explores actions without any overall purpose or \textit{insight}. With the guidance from the Model Network, the Q-Network explores actions that would either contradict or confirm an overall hypothesis about the nature of the learning environment.

\begin{figure}[h]
\includegraphics[width=10.2cm]{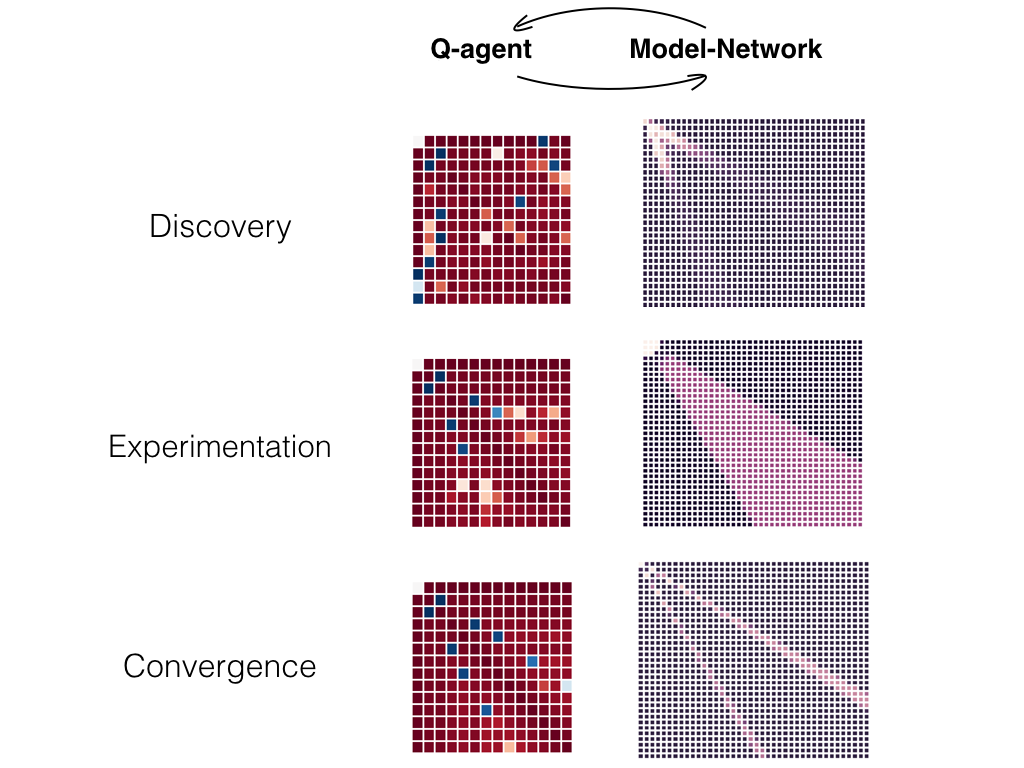}
\centering
\caption{HQN in various stages of training on Wythoff's game. The early models (Discovery) will be largely unsuccessful, while certain inaccurate generalizations (Experimentation) will supply reasonable strategies to the Q-Network, allowing the provision of useful datasets into the model network that translate into accurate and general models (Convergence) The Q-network has been trained for 2000 gameplays across each time-step.
}
\label{fig:ModelStages}
\end{figure}

\par Wythoff's games have the type of mathematical structure that should be very easy for a neural network to recognize, explaining a significant portion of the HQN agent's success. Unfortunately, neural networks are less adept in recognizing discrete, stepwise patterns then they are in recognizing regions and finding slopes. For example, even the best models generated by the HQN agent for Wythoff's game largely ignored the stepwise distribution of the cold positions across the line. As previously mentioned, this issue begs the existence of a layer that can be more flexible in the types of models that it could hypothesize.

\subsection{HQN Efficiency vs. Q-agent and Q-Networks}

We compared the performance of the HQN agent in Wythoff's game to that of a Q-agent and a Q-Network (QN) . Figure~\ref{fig:comparison} shows the accuracy of all three agents during learning. The HQN agent improves performance in discrete jumps as better models replace worse ones over time. Since models are assessed by the HQN in an unsupervised manner, some models evaluated to be better will in fact be less accurate, explaining the occasional fall in the performance of the HQN agent.

\begin{figure}[h]
\includegraphics[width=12cm]{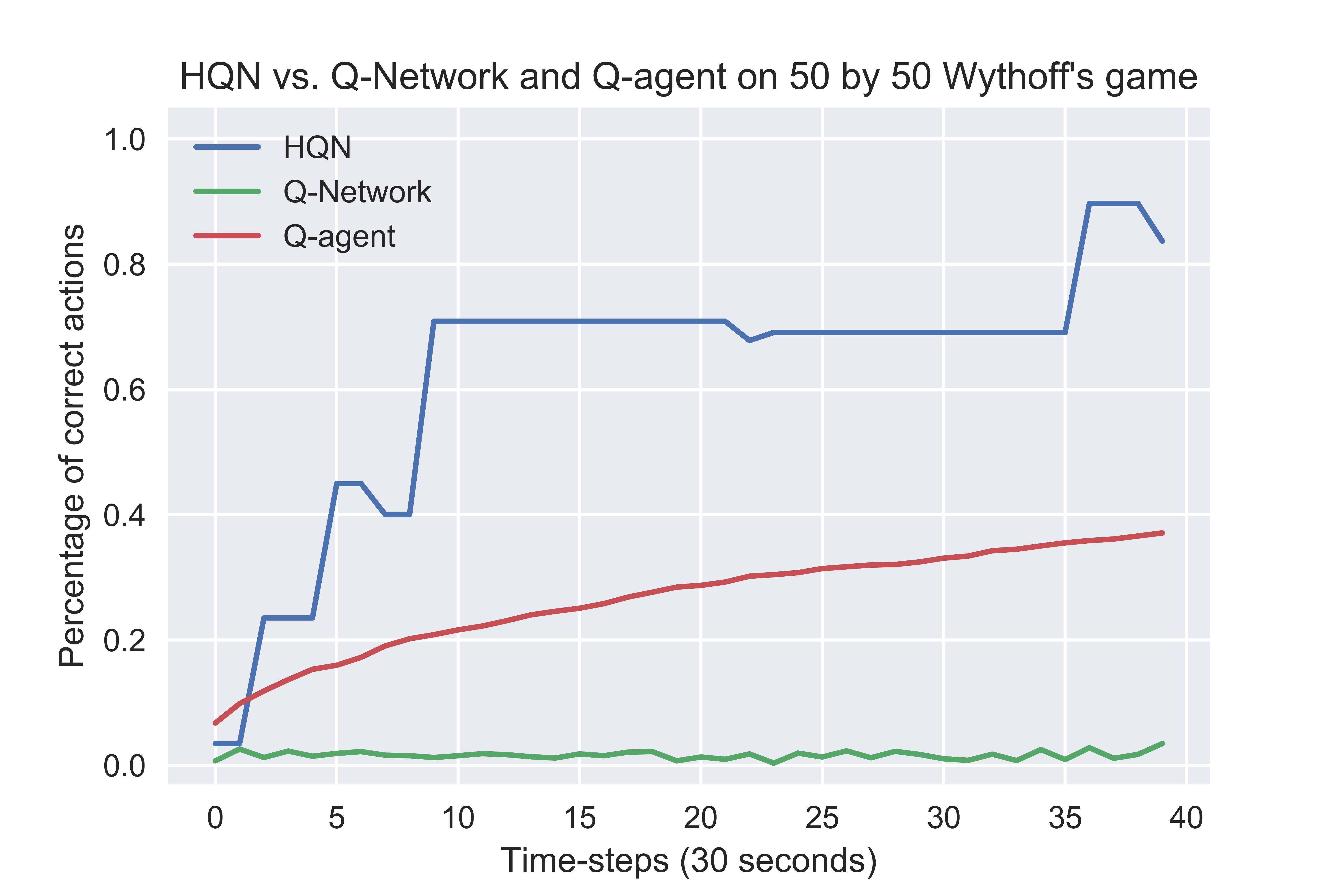}
\centering
\caption{The HQN vastly outperforms the Q-agent and the Q-Network on identical training periods with respect to time. Time-steps calculated on a 2.7 GHZ Intel i5 2-core processor (2015 MacBook Pro). The Naïve-Q is trained on a 50 by 50 board on 5000 gameplay simulations per time-step, whereas the Q-Network is trained on a 12 by 12 board with 1000 simulations each. The HQN learns better models in discrete jumps whereas the Q-agent has a steady but diminishing learning rate. Q-agent parameters: ($\lambda=1$, $\alpha=0.1$, $\beta=0.6$). Q-Network parameters: (single-layer with 15 neurons, standard backpropagation, fully-connected, sigmoid activation function, sum-squared-error cost function, $\alpha=0.01$, $\beta=0.7$) Change in network architecture, cost functions, or layer count did not have a noticable effect on the learning outcomes of the Q-Network agent.
}
\label{fig:comparison}
\end{figure}

\par The core idea that gives rise to the Q-Network is using neural networks to approximate the Q-function. Whereas a traditional Q-learning  attempts to fill in every single value for the Q-function in increasing accuracy in a look-up table manner, a Q-Network attempts to train a neural network that approximates this function. The Q-Network is also able to interpolate after training, since the network attempts to approximate the Q-function continuously, filling in for the gaps in the dataset. The benefits of such an approach have been demonstrated in detail in DeepMind's Atari Network \cite{atari}. However, while attempting to be more efficient and general than a naive Q-agent, Q-Networks  sacrifice a lot of stability. Re-training a Q-Network with a newly discovered dataset can be destructive to already existing features of the network. In order to remedy the destructive re-training issue, especially while training through large datasets, (deep) Q-Networks make use of ``experience replay'' \cite{atari}. A Q-Network agent that uses experience replay will store training data as it comes, and backpropagates that data across the network in occasional intervals, as opposed to some novel data-point.

\par For Wythoff's game, the Q-Network agent's performance was subpar compared to the HQN and naive Q-agents, even with additional modifications such as experience replay. Overall Q-Network performance did not exceed random chance significantly within the time constraints that allowed the HQN and the Q-agent to attain reasonable performance. Giving the Q-Network additional advantages, such as training against a perfect agent, or increasing the number of layers in the neural network, was not able to fix the disparity.

\par The only structural change that observably changed the behavior of the Q-Network was to equate actions and states in the training phase. In an impartial game, how good an action is depends only on which state the action takes the game to. Moves towards cold positions are good moves, whereas moves towards hot positions are bad moves. Under most learning tasks, this assumption does not hold, e.g. pressing left could win the game in a certain scenario, but be disastrous in the other. As illustrated in Figure~\ref{fig:qnetworkalgorithm}, the Q-Network, similar to the HQN and the Q-agent, does not operate under this assumption, since the network trains to approximate how good an action is given the state. $$action=\Delta(state)$$ However, we can hard-code the irrelevance of the starting state as a assumption, by representing an action as simply the encoding of the new state.
$$action=state_{new} $$ Under this framework, the task of the Q-Network would be to output the identical Q-vector that separates good states from bad states, given any state in the game. Since there are a lot more states in a game of Wythoff (Order=$\mathcal{O}(rows\cdot cols)$) then there are actions (Order=$\mathcal{O}(rows+cols)$, assuming $action=\Delta$(state)), the resulting Q-vector will be significantly larger. Increasing the number of neurons to the same order fixes the problem while slowing down training periods. However, the resulting modified agent is able to converge on strategies on a pace that is competitive.

\par This ``trick'', however, is inapplicable in most scenarios outside impartial games, that is why we did not hard-code such notions to the HQN agent. For a similar reason, we do not include the modified Q-Network agent that treats an $action$ as a $state_{new}$ in our analysis.

\subsection{HQN performance across dimensions}

\begin{figure}[h]
\includegraphics[width=10cm]{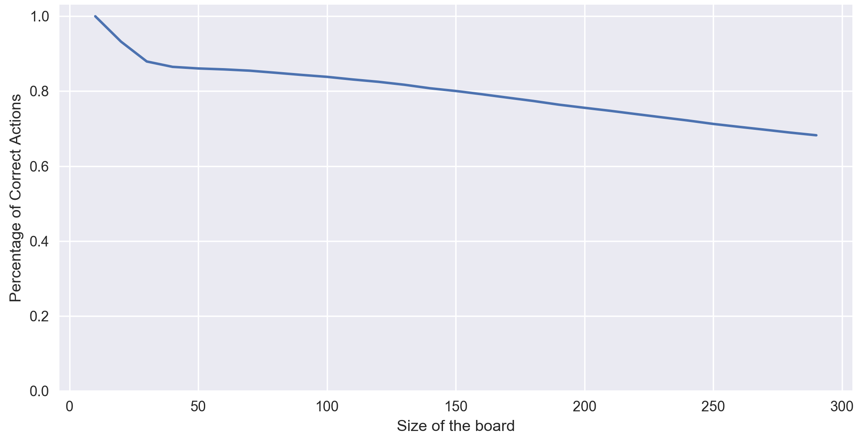}
\centering
\caption{Performance of a Model Network that the HQN agent converged to through datasets received from a Q-Network trained on a 12 by 12 board. Single layer feedforward network with cross-entropy cost function, trained for 2000 epochs, on a dataset extracted from the Q-Network of size 78. The HQN agent gave scores ranging between 0.95 and 0.97 to this model by benchmarking performance against a greedy Q-agent on a 50 by 50 board.
}
\label{fig:dimensions}
\end{figure}

\par The HQN agent was able to attain reasonably high levels of performance beyond the dimensions it was trained in. Figure~\ref{fig:dimensions} shows the accuracy of a specific model generated by the HQN agent for Wythoff's game across different board dimensions.

\par The Q-agent was collecting data on a 12 by 12 board, while the Model-Network was evaluating generated models against the Q-agent on a 50 by 50 board. The fact that the Model-Network tests models by their performance on dimensions that they were not trained on is crucial to prioritize generalizability across dimensions. The fact that the Q-agent cannot perform optimally on higher dimensions is also an advantage, since  models that achieve some level of generalizability will be assigned a higher score despite having poor accuracy on smaller boards.

\par Fortunately, even though the Model-Network was attempting to optimize performance up to a 50 by 50 board size, the models generated were able to display a reasonable degree of performance on boards that are larger. For the model in Figure~\ref{fig:dimensions}, the model achieved 70\% accuracy on a board that was 300 by 300.

\subsection{HQN performance across rule variations}

\begin{figure}[h]
\includegraphics[width=12cm]{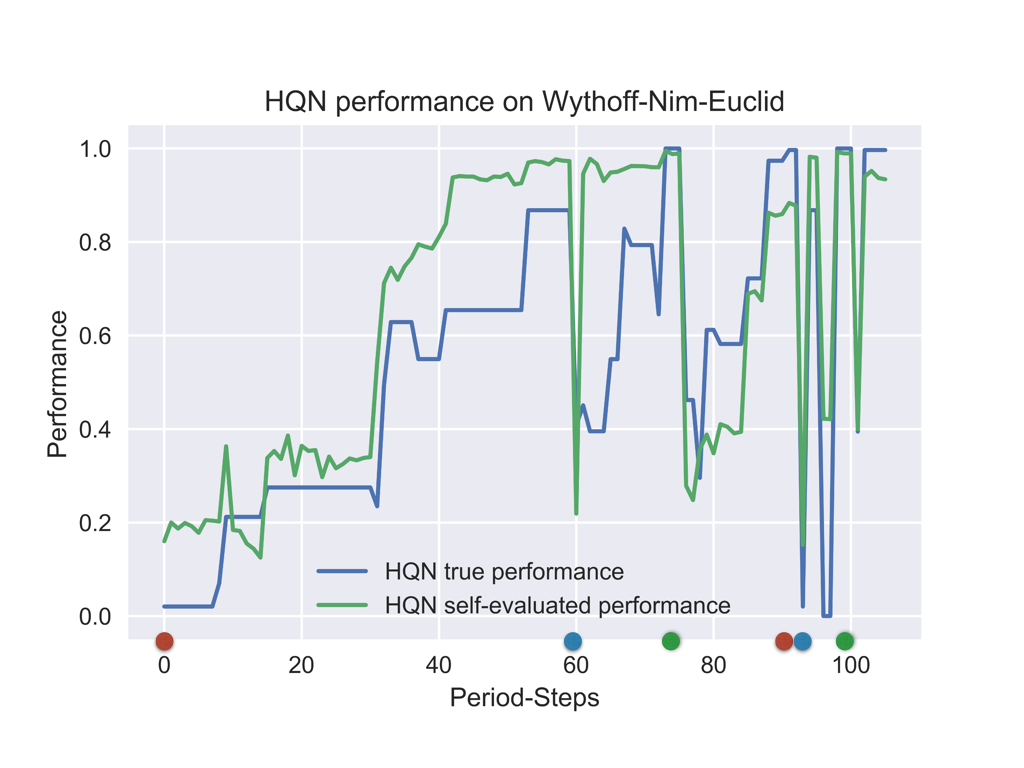}
\centering

\caption{Performance of a HQN across 3 different impartial combinatorial games. When the HQN converges on a model with performance exceeding a certain threshold, the rules of the game are changed. The HQN discovers the changes in rules, and adapts and reuses old models if they apply to the new set of rules. Performance decreases become less drastic as HQN learns all three games simultaneously. Tasks are cycled through in the order Wythoff-Nim-Euclid, shown as red-blue-green circles respectively in the x-axis. In a period, the Q-agent component of the HQN is trained across 250 gameplays for Wythoff, and 50 gameplays for Euclid and Nim.}
\label{fig:rules}
\end{figure}

HQN was also benchmarked against contexts where the rules of the game did not stay constant. In Figure~\ref{fig:rules} shows performance of the HQN agent across three games that had similar, but not identical, rules. The HQN agent started out by being trained through Wythoff's game. Once the agent reached satisfactory ($>0.9$) performance, we changed the gameplay rules to that of the game Nim, without informing the agent of this change. The agent was able to detect this change through the sharp decrease in the model's performance (as perceived by the HQN agent, displayed with the green line), store the model for the Wythoff task away, reset the Q-agent, and start training again. Since Nim (and later, Euclid) has less complex of an action space, we decreased the period from 250 Q-agent gameplays to only 50, in order to slow the learning process down for visualization purposes. When satisfactory performance was reached in Nim, we changed the learning task to Euclid. When we cycled through the three tasks in a similar fashion once again; however, the agent was able to attain satisfactory performance immediately after it detects a change in the rules.

\section{Discussion}

\par In this paper, we proposed some basic strategies for developing and evaluating agents that learn adaptable and robust strategies, an increasingly important goal for developing AI capable of navigating novel environments. The hierarchical structure of HQN showed promise in the transfer-learning domain, while remaining competitive with standard RL approaches in terms of performance. We trained a Q-agent, a Q-Network, and a HQN for identical amounts of time on the Wythoff's game (see Figure~\ref{fig:comparison}). The Q-agent was able to show improved accuracy, although at a steadily decreasing rate over time. The Q-Network, a more unstable but also more efficient advancement over the Q-agent algorithm, was not able to learn as well in this context of impartial games. The HQN agent, on the other hand, achieved increasing accuracy in discrete jumps as better models for the environment were discovered. The HQN also did more than merely excel in terms of efficiency. In the transfer-learning domain, where standard RL approaches are infamously unsuccessful, the HQN agent was able to achieve performance that generalized across dimensions (Figure~\ref{fig:dimensions}) and remain resistant to changes to rules of the game (Figure~\ref{fig:rules}). Most importantly, we could query the HQN agent to show its strategy for game play in an intuitive and explainable way.

\par Towards this goal of extensibilty in artificial agents, meta-reinforcement learning, the idea that RL agents can be trained to build better base networks for other RL agents to be trained on, holds a lot of promise. Wang et al. \cite{learningtolearn}, Duan et al. \cite{rl2} and Hansen \cite{deepepisodic} provide state-of-the-art approaches to Meta-reinforcement learning, that they call Deep Meta-Reinforcement Learning (DMRL), $RL^2$, and Deep Episodic Value Iteration (DEVI) respectively. These agents are evaluated against benchmarks beyond efficiency and accuracy metrics, including one-shot changes to rewards, and ability to learn abstract task structure. Real et al. \cite{imageclassifiers} and Miikkulainen et al. \cite{miikkulainen} also propose algorithms to optimize network architecture, including connectivity and parameters, for high-dimensional deep learning tasks such as image recognition and language modeling. We consider these efforts important as we aim for artificial networks that can generalize across tasks and yield interpretable outcomes.

\subsection{The Hierarchical Q Network}

\par Our key innovation in this study was the introduction of the  Hierarchical Q-Network (HQN), a \textit{model-based} learning agent that capitalizes on hierarchical information processing (see discussion of biological motivations in subsection \ref{sssec:num43}). The HQN was composed of an ``lower'' layer, the Q-agent, that explored through the high-dimensional state-action search space, and an ``higher'' layer, the Model-Network, that abstracted away the action dimension, and processed through the expected values of states to extract generalized structure from the environment. More important than the hierarchical structure of the HQN, however, is that the two networks interact in such a way that observations by the Q-agent effectively inform model building and that hypotheses generated by the Model Network effectively constrain future action policies. Without the Model-Network, the Q-agent blindly explores the massive search space without any ``insight''. Conversely, without the Q-agent, the Model-Network does not have any information with which to generalize from.

\par While the HQN's performance was superior than the other RL agents tested here, we should point out that it does suffer from limitations that future work should focus on. One of the inherent limitations of the HQN agent as proposed in this paper is that neural networks were used as the implementation of the Model-Network. Neural networks proved themselves to be suitable in a wide variety of learning tasks; however, there exists a wide range of limitations. For instance, a standard neural network will not be able to classify objects that follow a discrete pattern. For example, in Wythoff's game, even though the Model-Network was able to generate models that recognized the two symmetrical lines of cold positions, the network was unable to appreciate the discrete intervals separating cold positions along each of the lines. Processing-units that can independently and cohesively handle a vast array of decision problems are essential, if the goal is to understand and simulate how the biological brain can seamlessly navigate a highly-complex physical environment where inputs and goals of a learning task can change rapidly. We propose that symbolic representations combined with the strengths of statistical approaches of neural networks might be extremely useful. An initial attempt to explore such an intersection is given by Garnelo et al. where they propose a symbolic model-based learning agent \cite{garnelo}. We intend to follow a similar direction in our future work.

\subsection{Biological Motivations for Hierarchical Processing} \label{sssec:num43}

\par The advantages of the HQN, along with recent work by others \cite{lake2017building, atari, deepepisodic}, suggests that hierarchical structure is an effective catalyst for adaptive and generalizable learning in artificial agents. Indeed, substantial evidence from experimental and computational neuroscience suggests the same is true of biological brains \cite{doll, frank, badre}, pointing towards the looped architecture of cortico-basal ganglia networks as an important feature for model-based and model-free learning systems \cite{frank, badre}. The basal-ganglia (BG) is a subcortical network that receives widespread cortical input through the striatum, forming a channel-like architecture - each channel representing a particular action - that loops back up to motor cortex through the thalamus \cite{alexander}. Critically, each action channel in the BG contains a facilitation and suppression pathway, capable of exerting bidirectional control over the corresponding action channel in primary motor cortex. Schultz and colleagues \cite{schultz} famously showed that, during learning, the weights of these pathways are adjusted by phasic changes in striatal dopamine, encoding both the magnitude and sign of the prediction errors estimated from Q-learning models. This dopamine-dependent plasticity of cortico-striatal connections serves to reinforce the future selection of rewarding actions while also suppressing less desirable alternatives, serving a similar computational goal to that of Q-Networks \cite{carrot, cox, kravitz}. However, as previously mentioned, relying on feedback alone to drive learning 1) quickly becomes inefficient as task complexity increases, 2) limits the range of learned associations that can be simultaneously stored and exploited, and 3) fails to account for the robust and flexible nature of mammalian behavior.

\par The fundamental idea behind model-based learning is that, through experience and observation, internal beliefs are formed about the causal relationship between contextual features, states, and action values. For hierarchically structured tasks, for which state-action values depend on multiple, nested contextual features, generative models offer an imperfect but highly efficient strategy for guiding action selection. Critically, however, implementing a model-based learning strategy often relies on simultaneously learning from feedback in a model-free manner. Thus, the challenge of implementing model-based learning is two-fold, requiring 1) a generative mechanism for constructing hypotheses and 2) fluid interaction between inferential and feedback-dependent learning systems.

\par Both the neuroscience \cite{mcdannald2012model, dayan2014model, daw2011model} and machine learning \cite{hassabis, kool2016does, daw2014algorithmic, doya2002multiple} communities have shown a growing interest in model-based learning mechanisms , leading to mutually informative lines of investigation (e.g., understanding how biological brains encode model-based learning strategies provides hints for overcoming the challenges of model-based learning in artificial agents). Evidence from human neuroimaging studies suggests that model-free learning computations in the BG are regulated by top-down inputs from a model-based learning system in the prefrontal cortex (PFC) \cite{doll}. Critically, due to the looped architecture of cortico-BG pathways, model-based computations in cortex are informed by feedback-dependent updates in the action-value landscape. Over time, cortical model-based learning systems generate predictions based on model-free computations and, in turn, provide top-down constraints that regulate feedback sensitivity and decision policies in the BG. This symbiosis between BG- and PFC-dependent learning systems is mirrored in the HQN, with observed state-action values in the Q-Network facilitating better predictive models in the Model Network that, in turn, improve future performance through top-down constraints on action evaluation. This scaffolding of model-based and model-free learning computations accelerates the learning process by proactively testing different hypotheses about the rule structure of the task and constraining future decision policies as confidence increases about the fidelity of these expectations.

\subsection{Impartial Games as a Benchmark}

\par We should point out that we are not the first to observe and leverage the fact that the benchmark environment chosen profoundly influences the learning agents we design. Although the success of the DeepMind Atari Network was impressive, the benchmark featured implausible 2D environments through a third person perspective.  Kempka et al. developed VizDoom \cite{vizdoom}, a dynamic first-person perspective learning environment as an alternative testing benchmark for visual RL agents. $RL^2$ \cite{rl2} was evaluated in the VizDoom environment to demonstrate adaptability to high-scale problems. More recently, DeepMind and Blizzard announced a partnership \cite{starcraft} to utilize StarCraft II as a AI research environment. StarCraft II is a third-person strategy game with complicated raw visual input, state and action space, and delayed rewards and punishments to selected actions. Initial results already show that this new learning environment will be a challenge for even to most well-established deep reinforcement learning architectures.

\par We share similar goals with most of the aforementioned research, including designing learning agents that are more adaptable to changes in inputs and goals, as well as ensuring that learning outcomes are interpretable to humans. However, our critical argument is that, in order to achieve these goals, tasks should be designed in which adaptability, as opposed to accuracy, is prioritized. One of the ways our approach separates itself is in the sheer simplicity of the learning task chosen: Impartial games are equipped with rules straightforward enough that winning strategies have a complete mathematical theory. The fact that impartial games are ``solved'' games allows us to conveniently evaluate performance, and shift focus entirely to the transfer-learning and model-building domain.

\par Despite their simplicity, the scalability of impartial games makes them uniquely conducive to experimentation with model-based learning algorithms. Common benchmarks such as multi-armed bandit problems \cite{learningtolearn} lack an environment that needs to be navigated through dynamic model-building: a model for the environment cannot go beyond the predetermined expected value and variance distribution. The complexity of games like Go and Starcraft II, on the other hand, preclude any straightforward approach to model-building. For impartial games, model-building can be performed by exploring the geometrical structure of value over topology of the game environment. Thus, we argue that impartial games offer a more suitable environment for rigorously testing and comparing deep model-based agents. The benefits of using impartial games for benchmarking model-based deep learning are summarized below.

\begin{itemize}
  \item \textit{The rules of impartial games immediately generalize to bigger board dimensions, in a way that preserves the mathematical structure of the winning strategy.} This feature allows us to differentiate between learning agents beyond simply looking at their performance. In order to realize whether a learning agent has truly understood the nature of the game environment, we would just need to benchmark it on a bigger board size. Thus, the learning outcomes of an  agent become more transparent. Games like Chess or Go do not have structure that straightforwardly generalizes across different board sizes, so such an approach at benchmarking would have been infeasible.
  \item \textit{Impartial games offer wide variety of ways to change the rules of the game, without destroying the inherent mathematical structure of the winning-strategy associated with the specific set of rules.} Just by imposing some natural restrictions on the function that returns the set of legal moves from a position on Wythoff's game, we were able to generate two other games (Nim and Euclid) where the winning strategy has similar mathematical structure. Rule changes in games such as Chess or Go, albeit how insignificant, may influence overall strategy in very intricate ways. Thus Chess and Go would be less accessible for initial attempts for transfer-learning across rule changes.
  \item \textit{There are a lot of impartial games where structure and noise can co-exist, similar to the real world.} This is a feature that we did not utilize in this paper, but also reflects an advantage of impartial games. For example, a complete mathematical characterization of the winning (hot) and losing (cold) positions in a 3D Wythoff's game, as of time of writing, has not been discovered. However, results from the 2D version partially generalize to shed some light into optimal behaviour. A learning agent that can figure out how to make generalizations across dimensions could be worthwhile challenge.
\end{itemize}

All these features allow us to conclude that impartial games, when taken as a benchmark for learning agents, allows for asking questions to the agent \textit{where answers in the affirmative demonstrate a type of intelligence that goes beyond a brute-force pattern matching task}.

\newpage

\section{Appendix}
In this section, we formalize some definitions referred to in the rest of the paper. We also provide some of the proofs that will create suitable mathematical background for analysing impartial combinatorial games, such as Wythoff's game. We start by giving a formal description of an impartial game.
\begin{definition}\label{def:impartialgame}
Let $S$ be a set of states, and $f:S\rightarrow S$ be the legal moves function. An impartial game is a game played among $p_1$ and $p_2$, such that:
\begin{enumerate}
  \item $p_1$ and $p_2$ alternate in making moves with $p_1$ going first.
  \item Given a state $s$, the move $m$ made by $p_1$ has to be an element of $f(s)$
  \item $p_i$ loses at state $s$ if and only if it is $p_i$'s turn, and $f(s)$ is the empty-set, meaning that there are no legal actions for $p_i$ to do.
  \item There cannot exist a sequence of states $s_1, s_2, \cdots, s_n$ such that $f(s_1)=s_2$, $f(s_2)=s_3$, $\cdots $, $f(s_{n})=s_1$.
  \item From every state $s$ there exists a valid sequence of states $s_1, s_2, \cdots, s_n$ such that $f(s_1)=s_2$, $f(s_2)=s_3$, $\cdots $, $f(s_{n-1})=s_n$
  where $f(s_n)$ is the empty-set. Thus, $s_n$ is a terminal state in the game.
\end{enumerate}
\end{definition}
Conditions $1$, $2$ and $3$ lay out the main structure of the game. Condition $4$ insists that once a state has been reached, it cannot be re-accessed, and thus the game cannot go in cycles. Condition $4$, combined with Condition $5$ ensures that the game will always terminate, since every legal move must decrease the maximum distance to a terminal state where there are no available actions. Once the the distance reaches zero, the player whose turn it is loses, and the other player wins.
\par First, we prove using the principle of mathematical induction that indeed, in any well-defined impartial game, every position will be either hot or cold. First, we formally define the notions of \textit{hot} and \textit{cold}.
\begin{definition}\label{def:hotcold}
  Let $(S, f)$ be an impartial game. Let $s\in S$. We say $s$ is cold if $f(s)=\emptyset$, that is, $s$ is a terminal position. We say $s\in S$ is hot if and only if there exists a $s\in f(s)$ such that $s'$ is a cold position. If $s$ is not a cold position, we say $s$ is cold if and only if for all $s'\in f(s)$, $s'$ is a hot position.
\end{definition}
Thus, the definition of hot and cold recursively builds up from each other, and since terminal states being \textit{cold} constitute the necessary base case, the recursion is well-defined. However, if the reader is unfamiliar with recursive constructions, the theorem we present next does not immediately follow.
\begin{theorem}\label{thm:hotcold}
  Let $(S, f)$ be a impartial game. Then, for all $s\in S$, $s$ is either hot or cold.
\end{theorem}
\begin{proof}
  \par Proof is by induction on the maximum distance the state has to a terminal state. We refer to this distance as the depth of the state.
  \par In the base case, the depth is just $0$, implying that the $s$ is a terminal state, then by definition $s$ is cold, so the theorem is true.
  \par Now, we assume inductively that the depth of $s$ is greater than $0$. For all $s'\in f(s)$, the depth of $s'$ in necessarily smaller than that of $s$, thus inductive hypothesis applies to show all such $s'$ is either $hot$ or $cold$. \\[0.05in]
  \textbf{Case 1: } All such $s'$ are $hot$. \\
  Then by definition, $s$ is cold.\\[0.05in]
  \textbf{Case 2: } There exists a $s'$ which is $cold$. \\
  Then by definition, $s$ is hot.
  \par Since these are the only two cases, the result follows by induction.
\end{proof}
Thus we have that for all impartial games, the states can be partitioned into hot and cold positions. The question that remains is what the partition is, given a specific impartial game. We answer this question for the game of Nim. We will prove that a position in $2$ pile Nim is hot as long as the piles are asymmetrical.
\begin{theorem}
  Let $(S, f)$ be the impartial game of Nim restricted to only 2 piles. Then, $(a,b)\in S$ is a cold position if and only if $a=b$.
\end{theorem}
\begin{proof}
  Proof is by induction on depth of the state $(a,b)$. If depth is $0$, we know $(a,b)=(0,0)$, $(a,b)$ is cold, and $a=b$, as desired. If the depth is greater than $0$, we have two cases. \\[0.05in]
  \textbf{Case 1: } $a\neq b$ \\
  Without loss of generality, assume $a>b$. By the rules of Nim, there exists a move that decreases $a$ to $b$. Since $(b,b)$ has a smaller depth, by induction, $(a,b)$ is hot, as desired. \\[0.05in]
  \textbf{Case 2: } $a=b$ \\
  In this case, since in Nim diagonal moves are disallowed, all moves will will bring the game to a state where the piles are asymmetrical. Any new state will have a smaller depth, and by induction, will be hot. Thus, $(a,b)$ is cold, as desired.
\end{proof}
The partition in Nim for arbitrarily many dimensions has a similar structure, but requires a little more background to  prove, hence but we state it below.
\begin{theorem}
  Let $(n_1, n_2, \cdots, n_k)$ represent a $k-dimensional$ Nim game. Then, $(n_1, n_2, \cdots, n_k)$ is a cold position if and only if when $n_1$ through $n_k$ combined with the bitwise exclusive or operation (xor), the result is $0$.
\end{theorem}
\par Since bitwise logical operators bring us into the realm of stepwise distributions again, it becomes difficult for a HQN-like agent to converge on optimal performance.
\par For Wythoff's game, the proof for the partition is again somewhat involved, and hence we omit a proof for Theorem \ref{thm:wythoff}. We do present, however, a short proof for the partition of the hot and cold positions for Euclid, making use of the properties of the golden ratio.
\begin{theorem}
  Let $(S, f)$ be the impartial game of Euclid restricted to $2$ dimensions. Then, let $(a,b)\in S$, and without loss of generality, assume $a\leq b$. Then, $(a,b)$ is hot if and only if $a>\frac{1}{\phi}\cdot b$ where $\phi=\frac{1+\sqrt{5}}{2}$ is the golden ratio.
\end{theorem}
\begin{proof}
   Given a game state $(a,b)$, it suffices to show (1) that if $a\leq\frac{1}{\phi}\cdot b$, then $a$ is a $cold$ position, and (2) otherwise $(a,b)$ is a $hot$ position. Since we reduce one of the dimensions each move, and theorem works for terminal positions trivially, we can inductively assume theorem works for all accessible states from a given state.
   \par Let $(a,b)$ be a game state for Euclid, and suppose without loss of generality that $a\leq b$.
   \par \textbf{(1)} First, let $a<\frac{1}{\phi}\cdot b$. Since $ \frac{1}{\phi}>\frac{1}{2}$, the only state accessible from $(a,b)$ is $(b-a, a)$. We need to show $b-a>\frac{1}{\phi}\cdot a$, which implies by inductive hypothesis that $(b-a, a)$ is a $hot$ position.
   \begin{align*}
     b-a &> \phi \cdot a - a &&\text{Since $\phi\cdot a<b$ by assumption}\\
         &= (\phi-1)a \\
         &= \frac{1}{\phi}\cdot a &&\text{$\phi -1=\frac{1}{\phi}$ by definition of the golden ratio}
   \end{align*}
   \par \textbf{(2)} Now, we let $a>\frac{1}{\phi}\cdot b$. We want to access a game in the form $(a,b-qa)$ where $q$ is an integer, and $((a,b-qa))$ is a $cold$ position. By inductive hypothesis, this is equivalent to saying $a>\frac{1}{\phi}\cdot(b-qa)$ or $(b-qa)>\frac{1}{\phi}\cdot(a)$, depending on whether $a$ or $(b-qa)$ is the larger integer.
   \par The only reason why we would be unable to access such a state is while removing multiples of $a$ from $b$, we skip over the entirety of the $cold$ range. This would only happen if $a$ was a number larger then the number of unique $k$'s such that $(a,k)$ is a $cold$ position, i.e. $k<\frac{1}{\phi}\cdot a$ or $a<\frac{1}{\phi}\cdot k$ hold true. Combining the inequalities, we see that we need to count the number of integer $k$'s such that $$\frac{1}{\phi}\cdot a<k<\phi\cdot a$$ There will be precisely $a$ such values for $k$, so the entirety of the $cold$ range cannot be leaped over, as desired.
\end{proof}

\begin{figure}
\centering
\begin{subfigure}{.5\textwidth}
  \centering
  \includegraphics[width=.93\linewidth]{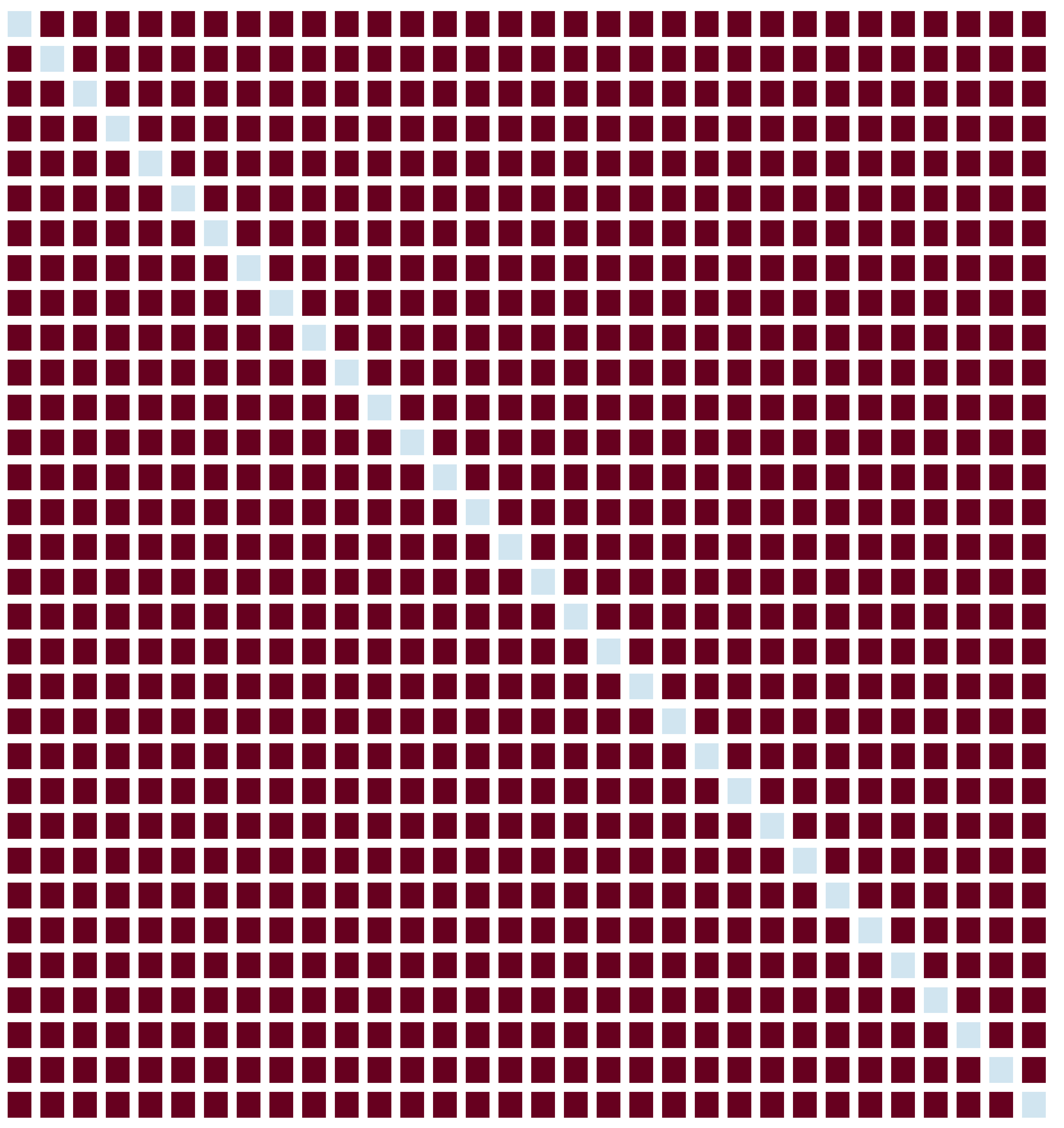}
  \caption{Nim}
  \label{fig:sub1}
\end{subfigure}%
\begin{subfigure}{.5\textwidth}
  \centering
  \includegraphics[width=.93\linewidth]{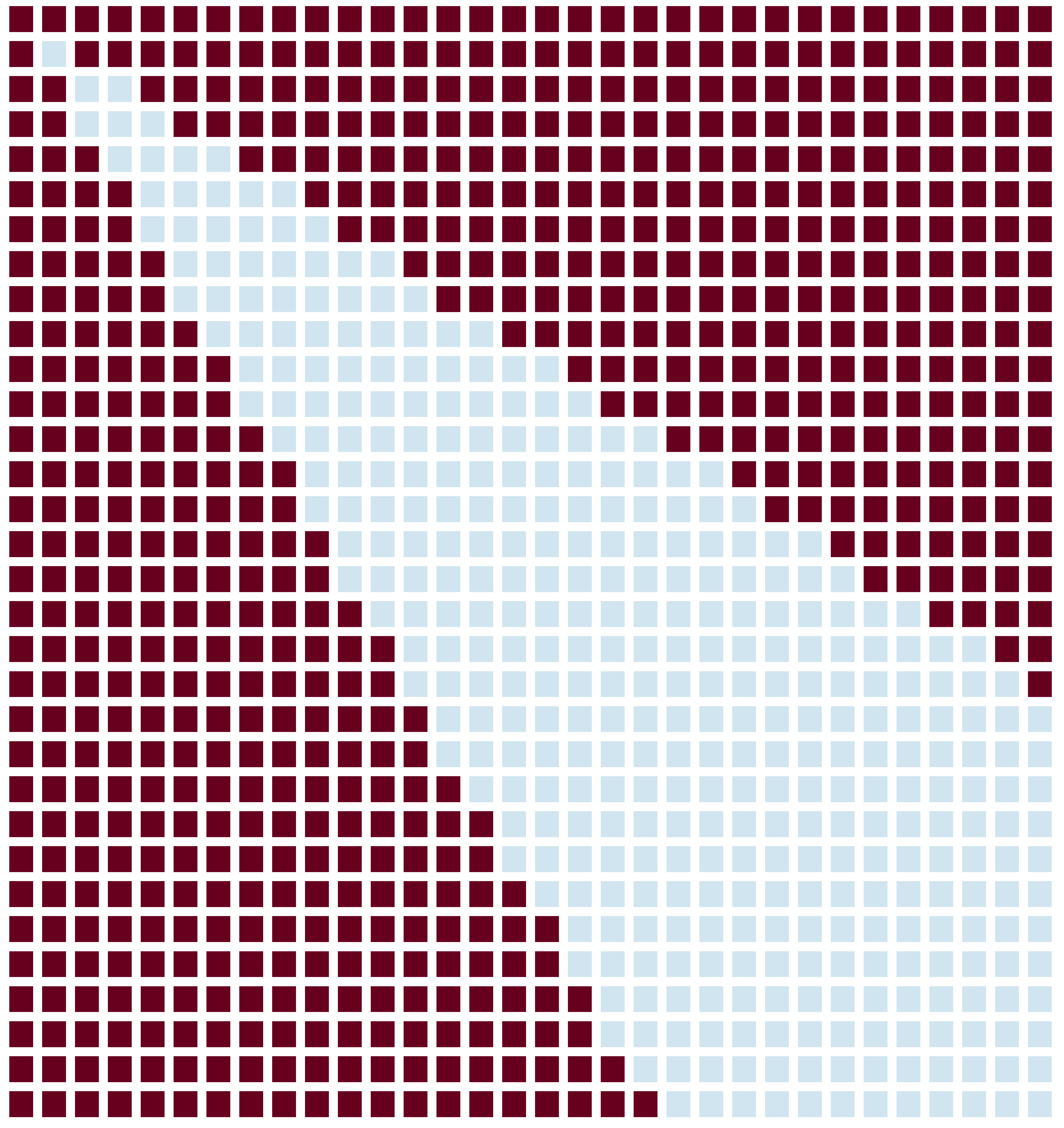}
  \caption{Euclid}
  \label{fig:sub2}
\end{subfigure}
\caption{$Hot$-$Cold$ Partitions for the impartial games Nim and Euclid. The mathematical rules that generate the partitions is similar to that of Wythoff's game; however, Nim and Euclid lack stepwise distribution along a region, unlike Wythoff's game.}
\label{fig:nimeuclid}
\end{figure}


\newpage

\begin{thebibliography}{10}

\bibitem{objectrec}
Yangqing Jia, Evan Shelhamer, Jeff Donahue, Sergey Karayev, Jonathan Long, Ross
  Girshick, Sergio Guadarrama, and Trevor Darrell.
\newblock Caffe: Convolutional architecture for fast feature embedding.
\newblock In {\em Proceedings of the 22nd ACM international conference on
  Multimedia}, pages 675--678. ACM, 2014.

\bibitem{speechrec}
Alex Graves, Abdel rahman Mohamed, and Geoffrey Hinton.
\newblock Speech recognition with deep recurrent neural networks.
\newblock pages 2--5, 2013.

\bibitem{atari}
Volodymyr Mnih, Koray Kavukcuoglu, David Silver, Andrei~A. Rusu, Joel Veness,
  Marc~G. Bellemare, Alex Graves, Martin Riedmiller, Andreas~K. Fidjeland, and
  Georg Ostrovski.
\newblock Human-level control through deep reinforcement learning.
\newblock {\em Nature}, 518:529--533, 2015.

\bibitem{alphago}
David Silver, Aja Huang, Chris~J. Maddison, Arthur Guez, Laurent Sifre George
  Van~Den Driessche, Julian Schrittwieser, Ioannis Antonoglou, Veda
  Panneershelvam, and Marc Lanctot.
\newblock Mastering the game of go with deep neural networks and tree search.
\newblock {\em Nature}, 529(7587):484--489, 2016.

\bibitem{lake2017building}
Brenden~M Lake, Tomer~D Ullman, Joshua~B Tenenbaum, and Samuel~J Gershman.
\newblock Building machines that learn and think like people.
\newblock {\em Behavioral and Brain Sciences}, 40, 2017.

\bibitem{toyama2017simple}
Asako Toyama, Kentaro Katahira, and Hideki Ohira.
\newblock A simple computational algorithm of model-based choice preference.
\newblock {\em Cognitive, Affective, \& Behavioral Neuroscience}, pages 1--20,
  2017.

\bibitem{kool2016does}
Wouter Kool, Fiery~A Cushman, and Samuel~J Gershman.
\newblock When does model-based control pay off?
\newblock {\em PLoS computational biology}, 12(8):e1005090, 2016.

\bibitem{marcus2018deep}
Gary Marcus.
\newblock Deep learning: A critical appraisal.
\newblock {\em arXiv}, 2018.

\bibitem{leike}
J.~Leike, M.~Martic, V.~Krakovna, P.~A. Ortega, T.~Everitt, and A.~Lefrancq.
\newblock Ai safety gridworlds.
\newblock {\em arXiv}, 2017.

\bibitem{lieder2015use}
Falk Lieder and Thomas~L Griffiths.
\newblock When to use which heuristic: A rational solution to the strategy
  selection problem.
\newblock In {\em CogSci}, 2015.

\bibitem{berlekamp}
Guy~R. Berlekamp, E. and J.~Conway.
\newblock {\em Winning Ways for your Mathematical Plays}.
\newblock A K Peters, Natick, MA, 1982.

\bibitem{doll12}
B.~B. Doll, D.~A. Simon, and N.~D. Daw.
\newblock The ubiquity of model-based reinforcement learning.
\newblock {\em Current opinion in neurobiology}, 22(6):1075--1081, 2012.

\bibitem{smittenaar}
P.~Smittenaar, T.~H. FitzGerald, V.~Romei, N.~D. Wright, and R.~J. Dolan.
\newblock Disruption of dorsolateral prefrontal cortex decreases model-based in
  favor of model-free control in humans.
\newblock {\em Neuron}, 80(4):914--919, 2013.

\bibitem{wunderlich}
K.~Wunderlich, P.~Smittenaar, and R.~J. Dolan.
\newblock Dopamine enhances model-based over model-free choice behavior.
\newblock {\em Neuron}, 75(3):418--424, 2012.

\bibitem{doll}
B.~B. Doll, K.~D. Duncan, D.~Simon, D.~Shohamy, and N.~D. Daw.
\newblock Model-based choices involve prospective neural activity.
\newblock {\em Nature Neuroscience}, 18:1--9, 2015.

\bibitem{russek17}
E.~M. Russek, I.~Momennejad, M.~M. Botvinick, S.~J. Gershman, and N.~D. Daw.
\newblock Predictive representations can link model-based reinforcement
  learning to model-free mechanisms.
\newblock {\em PLOS Computational Biology}, 13:9, 2017.

\bibitem{odoherty17}
J.~P. O'Doherty, J.~Cockburn, and W.~M. Pauli.
\newblock Learning, reward, and decision making.
\newblock {\em Annual review of psychology}, 68:73--100, 2017.

\bibitem{frank}
M.~J. Frank and D.~Badre.
\newblock Mechanisms of hierarchical reinforcement learning in corticostriatal
  circuits 1: Computational analysis.
\newblock {\em Cerebral Cortex}, 22:509--526, 2012.

\bibitem{badre2009}
D.~Badre and M.~D'esposito.
\newblock Is the rostro-caudal axis of the frontal lobe hierarchical?
\newblock {\em Nature Reviews Neuroscience}, 10(9):659--669, 2009.

\bibitem{schultz97}
W.~Schultz, P.~Dayan, and P.~R. Montague.
\newblock A neural substrate of prediction and reward.
\newblock {\em Science}, 275:1593--1599, 1997.

\bibitem{eshel2015arithmetic}
Neir Eshel, Michael Bukwich, Vinod Rao, Vivian Hemmelder, Ju~Tian, and Naoshige
  Uchida.
\newblock Arithmetic and local circuitry underlying dopamine prediction errors.
\newblock {\em Nature}, 525(7568):243--246, 2015.

\bibitem{eshel2016dopamine}
Neir Eshel, Ju~Tian, Michael Bukwich, and Naoshige Uchida.
\newblock Dopamine neurons share common response function for reward prediction
  error.
\newblock {\em Nature neuroscience}, 19(3):479--486, 2016.

\bibitem{daw2005}
N.~D. Daw, Y.~Niv, and P.~Dayan.
\newblock Uncertainty-based competition between prefrontal and dorsolateral
  striatal systems for behavioral control.
\newblock {\em Nature neuroscience}, 8(12):1704--1711, 2005.

\bibitem{badre}
D.~Badre and M.~J. Frank.
\newblock Mechanisms of hierarchical reinforcement learning in cortico-striatal
  circuits 2: evidence from fmri.
\newblock {\em Cereb. Cortex}, 22:527--36, 2011.

\bibitem{zach}
Zachary Abel.
\newblock Putting the why in wythoff.
\newblock
  \url{http://blog.zacharyabel.com/2012/06/putting-the-why-in-wythoff/}, 2014.

\bibitem{suttonbarto}
R.~S. Sutton and A.~G. Barto.
\newblock {\em Reinforcement learning: An introduction (Vol. 1, No. 1)}.
\newblock MIT press, Cambridge, 1998.

\bibitem{rumelhart}
D.~E. Rumelhart, G.~E. Hinton, and R.~J. Williams.
\newblock Learning internal representations by error propagation.
\newblock {\em California Univ San Diego La Jolla Inst for Cognitive Science},
  8506, 1985.

\bibitem{nimble}
Jorgen Grimnes.
\newblock Nimblenet.
\newblock \url{http://jorgenkg.github.io/python-neural-network/}, 2016.
\newblock Github repository.

\bibitem{learningtolearn}
J.~X. Wang, Z.~Kurth-Nelson, D.~Tirumala, H.~Soyer, et~al.
\newblock {\em Learning to Reinforcement Learn}.
\newblock 2017.

\bibitem{rl2}
Yan Duan, John Schulman, Xi~Chen, Peter~L. Bartlet, et~al.
\newblock $RL^2$.
\newblock {\em arXiv}, 2016.

\bibitem{deepepisodic}
Steven~S. Hansen.
\newblock Deep episodic value iteration for model-based meta-reinforcement
  learning.
\newblock {\em arXiv}, 2017.

\bibitem{imageclassifiers}
Esteban Real, Sherry Moore, Andrew Selle, Saurabh Saxena, et~al.
\newblock Large-scale evolution of image classifiers.
\newblock {\em arXiv}, 2017.

\bibitem{miikkulainen}
Risto Miikkulainen, Jason Liang, Elliot Meyerson, Aditya Rawal, et~al.
\newblock Evolving deep neural networks.
\newblock 2017.

\bibitem{garnelo}
M.~{Garnelo}, K.~{Arulkumaran}, and M.~{Shanahan}.
\newblock {Towards Deep Symbolic Reinforcement Learning}.
\newblock {\em ArXiv e-prints}, September 2016.

\bibitem{alexander}
G.~E. Alexander, M.~R. DeLong, and P.~L. Strick.
\newblock Parallel organization of functionally segregated circuits linking
  basal ganglia and cortex.
\newblock {\em Annu. Rev. Neurosci.}, 9:357--381, 1986.

\bibitem{schultz}
W.~Schultz, P.~Dayan, and P.~R.~A Montague.
\newblock Neural substrate of prediction and reward.
\newblock {\em Science}, 80:1593--1599.

\bibitem{carrot}
M.~J. Frank, L.~C. Seeberger, and R.~C. O'reilly.
\newblock By carrot or by stick: cognitive reinforcement learning in
  parkinsonism.
\newblock {\em Science}, 306:1940--3, 2004.

\bibitem{cox}
S.~M.~L. Cox et~al.
\newblock Striatal d1 and d2 signaling differentially predict learning from
  positive and negative outcomes.
\newblock {\em Neuroimage}, 109:95--101, 2015.

\bibitem{kravitz}
A.~V. Kravitz, L.~D. Tye, and A.~C. Kreitzer.
\newblock Distinct roles for direct and indirect pathway striatal neurons in
  reinforcement.
\newblock {\em Nature Neuroscience}, 15:816--8, 2012.

\bibitem{mcdannald2012model}
Michael~A McDannald, Yuji~K Takahashi, Nina Lopatina, Brad~W Pietras, Josh~L
  Jones, and Geoffrey Schoenbaum.
\newblock Model-based learning and the contribution of the orbitofrontal cortex
  to the model-free world.
\newblock {\em European Journal of Neuroscience}, 35(7):991--996, 2012.

\bibitem{dayan2014model}
Peter Dayan and Kent~C Berridge.
\newblock Model-based and model-free pavlovian reward learning: revaluation,
  revision, and revelation.
\newblock {\em Cognitive, Affective, \& Behavioral Neuroscience},
  14(2):473--492, 2014.

\bibitem{daw2011model}
Nathaniel~D Daw, Samuel~J Gershman, Ben Seymour, Peter Dayan, and Raymond~J
  Dolan.
\newblock Model-based influences on humans' choices and striatal prediction
  errors.
\newblock {\em Neuron}, 69(6):1204--1215, 2011.

\bibitem{hassabis}
D.~Hassabis, D.~Kumaran, C.~Summerfield, and M.~Botvinick.
\newblock Neuroscience-inspired artificial intelligence.
\newblock {\em Neuron}, 95:245--258, 2017.

\bibitem{daw2014algorithmic}
Nathaniel~D Daw and Peter Dayan.
\newblock The algorithmic anatomy of model-based evaluation.
\newblock {\em Phil. Trans. R. Soc. B}, 369(1655):20130478, 2014.

\bibitem{doya2002multiple}
Kenji Doya, Kazuyuki Samejima, Ken-ichi Katagiri, and Mitsuo Kawato.
\newblock Multiple model-based reinforcement learning.
\newblock {\em Neural computation}, 14(6):1347--1369, 2002.

\bibitem{vizdoom}
Michal Kempka, Marek Wydmuch, Grzegorz Runc, Jakub Toczek, and Wojciech
  Jaskowski.
\newblock {\em A Doom-based AI Research Platform for Visual Reinforcement
  Learning}.
\newblock 2016.

\bibitem{starcraft}
Oriol Vinyals, Timo Ewalds, Sergey Bartunov, Petko Georgiev, et~al.
\newblock
  \url{https://deepmind.com/blog/deepmind-and-blizzard-open-starcraft-ii-ai-research-environment/},
  2017.

\end{thebibliography}
\end{document}